\documentclass{article}
\usepackage{graphicx} 
\usepackage{amsfonts, amsmath, amssymb, amsthm}
\usepackage{mathtools}
\usepackage{enumitem}
\usepackage{hyperref}
\usepackage{rotating}
\usepackage{tikz}

\usetikzlibrary{arrows.meta,positioning}

\newtheorem{proposition}{Proposition}
\newtheorem{theorem}{Theorem}

\title{Thermodynamic Response Functions in Singular Bayesian Models}
\author{Sean Plummer \\ snplmmr@gmail.com}
\date{\today}

\begin{document}
\maketitle

\begin{abstract}

Singular statistical models—including mixtures, matrix factorization, and neural networks—violate regular asymptotics due to parameter non-identifiability and degenerate Fisher geometry. Although singular learning theory characterizes marginal likelihood behavior through invariants such as the real log canonical threshold and singular fluctuation, these quantities remain difficult to interpret operationally. At the same time, widely used criteria such as WAIC and WBIC appear disconnected from underlying singular geometry. We show that posterior tempering induces a one-parameter deformation of the posterior distribution whose associated observables generate a hierarchy of thermodynamic response functions. A universal covariance identity links derivatives of tempered expectations to posterior fluctuations, placing WAIC, WBIC, and singular fluctuation within a unified response framework. Within this framework, classical quantities from singular learning theory acquire natural thermodynamic interpretations: RLCT governs the leading free-energy slope, singular fluctuation corresponds to curvature of the tempered free energy, and WAIC measures predictive fluctuation. We formalize an observable algebra that quotients out non-identifiable directions, allowing structurally meaningful order parameters to be constructed in singular models. Across canonical singular examples—including symmetric Gaussian mixtures, reduced-rank regression, and overparameterized neural networks—we empirically demonstrate phase-transition-like behavior under tempering. Order parameters collapse, susceptibilities peak, and complexity measures align with structural reorganization in posterior geometry. Our results suggest that thermodynamic response theory provides a natural organizing framework for interpreting complexity, predictive variability, and structural reorganization in singular Bayesian learning.

\end{abstract}

\section{Introduction}

\subsection{Motivation}
Singular statistical models arise whenever multiple parameter values induce the same predictive distribution, as in mixture models, low-rank factorizations, and neural networks with symmetry or overparameterization \cite{watanabe2009algebraic, watanabe2018mathematics}. In such settings, the Fisher information may be degenerate and classical regular asymptotics fail: posterior mass concentrates on sets with nontrivial geometry, and standard ``effective dimension'' heuristics become unreliable.

Singular learning theory provides a principled asymptotic description of Bayesian marginal likelihood and predictive performance in these regimes through invariants such as the real log canonical threshold (RLCT) and the singular fluctuation. These quantities explain, for example, why the marginal likelihood can scale with sample size according to non-integer ``learning coefficients'' and why predictive complexity can differ from regular-model expectations. However, despite their theoretical importance, RLCT and related constants can be difficult to interpret operationally in finite samples and are rarely used as practical diagnostics.

At the same time, widely-used criteria such as the widely applicable information criterion (WAIC) and the widely applicable Bayesian information criterion (WBIC) are often applied in singular settings, yet their relationship to the underlying singular geometry can appear opaque \cite{watanabe2010asymptotic, watanabe2013widely}. This motivates the search for a structural and interpretable framework that (i) applies in singular models without requiring model-specific asymptotic derivations, and (ii) organizes practical complexity measures in a way that exposes what aspects of posterior geometry they are responding to. 

The central claim of this paper is that several widely used quantities in Bayesian model evaluation—including WAIC complexity and singular fluctuation—can be understood as thermodynamic response functions generated by posterior tempering. This work forms part of a broader effort to understand singular Bayesian learning through thermodynamic structure.

\subsection{Core idea: tempering as a deformation}
Our starting point is posterior tempering, which defines a one-parameter deformation of the posterior distribution:
\[
\pi_\beta(\theta \mid D) \;\propto\; \pi(\theta)\, p(D \mid \theta)^\beta,
\qquad \beta > 0.
\]
Varying $\beta$ interpolates between the prior ($\beta \to 0$) and the ordinary posterior ($\beta = 1$), producing a controlled family of distributions that reweights the same likelihood landscape \cite{neal2001annealed}. Tempering provides a particularly convenient deformation because it preserves the likelihood landscape while continuously reweighting its influence relative to the prior. This makes it possible to probe posterior structure without altering the underlying statistical model.

 We treat measurable functions $f : \Theta \to \mathbb{R}$ that are integrable under $\pi_\beta$ as \emph{observables}. By this we simply mean quantities whose posterior expectations are well-defined under tempering:
\[
\mathbb{E}_\beta[f] = \int f(\theta)\,\pi_\beta(\theta \mid D)\, d\theta.
\]
In physics terminology, observables are measurable quantities of a system; here they are posterior-measurable functions of model parameters. Their expectations summarize aspects of posterior geometry, and their fluctuations encode structural variability. A basic identity shows that the derivative of a tempered expectation is given by a covariance with the log-likelihood,
\[
\frac{d}{d\beta}\, \mathbb{E}_\beta[f]
\;=\;
\mathrm{Cov}_\beta\!\left(f, \ell\right),
\qquad
\ell(\theta) := \log p(D \mid \theta),
\]
so that sensitivities of observables with respect to $\beta$ are controlled by posterior fluctuations. This suggests a thermodynamic viewpoint: $\beta$ plays the role of an inverse temperature, $\ell$ plays the role of (negative) energy, and the resulting covariances and variances define \emph{response functions}. This identity implies that fluctuations of the log-likelihood control how
posterior expectations respond to changes in temperature, placing familiar Bayesian quantities within a thermodynamic response hierarchy.

\subsection{Contributions and paper organization}
This paper develops a response-function framework for singular Bayesian learning with three components.

\paragraph{Observable algebra.}
We formalize an ``observable algebra'' by quotienting out functions that are constant along non-identifiable directions (i.e., directions that do not change the induced predictive distribution). This separates parameterization artifacts from statistically meaningful quantities and guides the construction of order parameters in singular models.

\paragraph{Universal response identities.}
We use the covariance identity to define order parameters, susceptibilities, and complexity responses under tempering. Within this framework, WAIC and related quantities appear as fluctuation-based responses at $\beta=1$, while the curvature of a free-energy-like object provides a natural interpretation of singular fluctuation.

\paragraph{Empirical confirmation in canonical singular models.}
We empirically demonstrate that response functions localize structural reorganization in posterior geometry across three canonical settings: a symmetric Gaussian mixture (symmetry breaking), reduced-rank regression (algebraic singularity), and an overparameterized neural network (modern non-identifiability). Across these examples, order parameters collapse, susceptibilities exhibit sharp peaks, and complexity responses align with the same transition-like behavior.

\paragraph{Toward a thermodynamic framework for singular learning.}
The response perspective developed here is part of a broader effort to
understand singular statistical models through thermodynamic structure.
Posterior tempering introduces a natural control parameter whose associated
fluctuations reveal geometric and predictive properties of the model.
In this paper we focus on the response functions generated by this deformation
and their relationship to practical complexity measures such as WAIC and WBIC.
More generally, these ideas suggest that tools from statistical physics—
including response theory and renormalization concepts—may provide a useful
organizational language for singular Bayesian learning.

\paragraph{Organization.}
Section~\ref{sec:background} reviews tempering and minimal singular-learning-theory background. Section~\ref{sec:obs-algebra} introduces the observable algebra and examples. Section~\ref{sec:responses} develops response functions and places WAIC/WBIC within this hierarchy. Section~\ref{sec:nu} interprets singular fluctuation as a curvature response. Section~\ref{sec:experiments} presents experiments, and Section~\ref{sec:discussion} discusses implications and limitations.

\section{Background}
\label{sec:background}

\subsection{Singular models and posterior tempering}

Let $\Theta \subset \mathbb{R}^d$ denote a parameter space and let
\[
p(x \mid \theta), \qquad \theta \in \Theta,
\]
be a parametric statistical model. The model is said to be \emph{singular} if the map
\[
\theta \mapsto p(\cdot \mid \theta)
\]
is not locally injective and the Fisher information matrix may be degenerate at parameter values that realize the true distribution. Equivalently, distinct parameter values can induce the same predictive distribution, and the likelihood surface may exhibit flat directions or self-intersections. Mixture models, low-rank matrix factorizations, and neural networks with permutation symmetries are canonical examples.

Given data $D = \{x_i\}_{i=1}^n$ and prior $\pi(\theta)$, the tempered posterior is defined for $\beta > 0$ as
\[
\pi_\beta(\theta \mid D)
\;\propto\;
\pi(\theta)\, p(D \mid \theta)^\beta,
\qquad
p(D \mid \theta) = \prod_{i=1}^n p(x_i \mid \theta).
\]
The ordinary posterior corresponds to $\beta = 1$, while $\beta \to 0$ recovers the prior. Varying $\beta$ produces a one-parameter deformation that continuously reweights the likelihood contribution relative to the prior.

Define the partition function
\[
Z(\beta)
=
\int \pi(\theta)\, p(D \mid \theta)^\beta \, d\theta,
\]
and the associated free energy
\[
F(\beta)
=
- \frac{1}{\beta} \log Z(\beta).
\]
In regular models, derivatives of $\log Z(\beta)$ admit classical asymptotic expansions. In singular models, however, the geometry of the parameter-to-distribution map fundamentally alters the scaling behavior of $Z(\beta)$ with respect to sample size.

\subsection{Minimal results from singular learning theory}

Let $q(x)$ denote the true data-generating distribution. Define the
expected Kullback--Leibler divergence
\[
K(\theta)
=
\int q(x)\, \log \frac{q(x)}{p(x \mid \theta)}\, dx.
\]
Let
\[
\theta_\star \in \arg\min_{\theta \in \Theta} K(\theta)
\]
denote a minimizer of the KL divergence. In the realizable case,
$K(\theta_\star)=0$ and the set of minimizers may form a singular
variety rather than an isolated point.

Let the empirical negative log-likelihood be
\[
L_n(\theta)
=
- \frac{1}{n} \sum_{i=1}^n \log p(x_i \mid \theta).
\]

Under suitable regularity conditions, singular learning theory shows that
the marginal likelihood admits the asymptotic expansion
\[
\log Z(1)
=
- n L_n(\theta_\star)
+ \lambda \log n
- (m-1)\log \log n
+ O_p(1),
\]
where $\lambda$ is the \emph{real log canonical threshold} (RLCT) and
$m$ is its multiplicity \cite{watanabe2009algebraic}. In regular models, $\lambda = d/2$, while in
singular models $\lambda$ can be fractional and reflects the local
algebraic structure of $K(\theta)$ near its minimum set.

Another key quantity is the \emph{singular fluctuation} $\nu$, which
governs the leading-order fluctuation of the log-likelihood under the
posterior \cite{watanabe2010asymptotic}. It appears in predictive risk expansions and determines the
effective complexity of the model. In regular models, $\nu = d/2$, but
in singular models it generally differs from $\lambda$.

The widely applicable Bayesian information criterion (WBIC) evaluates
the tempered expectation of the log-likelihood at inverse temperature $\beta_n = 1/\log n$, providing an asymptotic approximation to the marginal likelihood in both regular and singular settings.

In what follows, we do not rely on detailed asymptotic expansions.
Instead, these quantities serve as structural anchors: the RLCT
describes leading-order free-energy scaling, singular fluctuation
controls curvature-like behavior of the tempered free energy, and WBIC selects a particular temperature regime.


\section{Observable Algebra}
\label{sec:obs-algebra}

\subsection{Observables and posterior expectations}

An \textit{observable} is a measurable function
\[
f : \Theta \to \mathbb{R}
\]
that is integrable under the tempered posterior $\pi_\beta(\theta \mid D)$ \cite{pathria2011statistical}.
Its expectation
\[
\mathbb{E}_\beta[f]
=
\int f(\theta)\,\pi_\beta(\theta \mid D)\, d\theta
\]
is therefore well-defined. Importantly, expectations depend only on the pushforward of
$\pi_\beta$ through the map $\theta \mapsto p(\cdot \mid \theta)$.
If two parameter values induce identical predictive distributions,
their contribution to posterior averages cannot be distinguished
through data alone. This motivates identifying observables that depend only on the
induced predictive distribution.

\subsection{Equivalence relation and induced distribution space}

Define an equivalence relation
\[
\theta \sim \theta'
\quad \Longleftrightarrow \quad
p(\cdot \mid \theta) = p(\cdot \mid \theta').
\]

Let $\widetilde{\Theta} = \Theta / \sim$ denote the set of equivalence
classes. Each element of $\widetilde{\Theta}$ corresponds uniquely to
a predictive distribution in the model family. The tempered posterior induces a probability measure on
$\widetilde{\Theta}$ via pushforward:
\[
\widetilde{\pi}_\beta(A)
=
\pi_\beta(\{\theta \in \Theta : [\theta] \in A\}).
\]
Thus posterior expectations of distribution-invariant observables
may equivalently be written as
\[
\mathbb{E}_\beta[f]
=
\int_{\widetilde{\Theta}} \tilde{f}([\theta]) \,
d\widetilde{\pi}_\beta([\theta]).
\]
No additional geometric structure on $\widetilde{\Theta}$ is required;
we use it only as a bookkeeping device for predictive equivalence.

\subsection{Distribution-invariant observables}

An observable $f$ is called \emph{distribution-invariant} if
\[
\theta \sim \theta'
\;\Rightarrow\;
f(\theta) = f(\theta').
\]
Such functions descend to well-defined measurable functions
\[
\tilde{f} : \widetilde{\Theta} \to \mathbb{R}.
\]
We refer to the collection of distribution-invariant observables
as the \emph{observable algebra}. Conceptually, these are measurable
functions on the space of predictive distributions. This construction removes variation arising purely from
non-identifiable parameter directions while preserving
all posterior-measurable predictive structure; paralleling the use of gauge-invariant observables in
statistical physics and quantum field theory, where physical quantities
are defined modulo redundant parameterizations.


\paragraph{From equivalence classes to the model image.}
Let
\[
\Phi : \Theta \to \mathcal{M},
\qquad
\Phi(\theta) := p(\cdot \mid \theta),
\]
and define the model image
\[
\mathcal{M} := \Phi(\Theta) = \{\, p(\cdot\mid\theta) : \theta \in \Theta \,\}.
\]
By construction, $\theta \sim \theta'$ if and only if $\Phi(\theta)=\Phi(\theta')$.
Hence there is a canonical bijection
\[
\iota : \widetilde{\Theta} \to \mathcal{M},
\qquad
\iota([\theta]) := \Phi(\theta),
\]
which is well-defined because $\Phi$ is constant on equivalence classes.

\begin{proposition}[Observable representation on $\mathcal{M}$]
\label{prop:observable-representation}
Let $f:\Theta\to\mathbb{R}$ be a measurable function.

\begin{enumerate}
\item If $f$ is distribution-invariant (i.e., $\theta\sim\theta'\Rightarrow f(\theta)=f(\theta')$),
then there exists a measurable functional $\psi:\mathcal{M}\to\mathbb{R}$ such that
\[
f(\theta)=\psi(\Phi(\theta))
\quad \text{for all } \theta\in\Theta.
\]
Equivalently, $f = \psi \circ \Phi$.
\item Conversely, any measurable $\psi:\mathcal{M}\to\mathbb{R}$ defines a distribution-invariant
observable $f=\psi\circ\Phi$ on $\Theta$.
\end{enumerate}
\end{proposition}

\noindent
Thus distribution-invariant observables are precisely measurable functionals of the induced
predictive distribution.

\begin{proof}[Proof sketch]
If $f$ is distribution-invariant, define $\tilde{f}([\theta]) := f(\theta)$ on $\widetilde{\Theta}$,
which is well-defined by invariance. Set $\psi := \tilde{f}\circ \iota^{-1}$ to obtain
$f = \psi\circ\Phi$. The converse is immediate since $\Phi(\theta)=\Phi(\theta')$ implies
$\psi(\Phi(\theta))=\psi(\Phi(\theta'))$.
\end{proof}

This quotient construction ensures that response functions depend only on predictive distributions and are therefore invariant to non-identifiable parameter directions.

\subsection{Identifiable versus statistically meaningful}

Identifiability is a local property: a parameter may be locally identifiable
in a neighborhood yet globally non-identifiable due to symmetry.
Conversely, some non-identifiable directions may carry no predictive
consequence.

The observable algebra isolates functions that are well-defined on the
space of predictive distributions. In regular models, the equivalence
classes are typically discrete and $\mathcal{O}$ coincides with the full
function space. In singular models, the quotient becomes nontrivial and
guides the construction of meaningful order parameters.

\subsection{Examples}

\paragraph{Symmetric Gaussian mixture.}
Consider a two-component Gaussian mixture with symmetric means
$\pm \mu$ and equal weights. The parameterizations $(\mu)$ and $(-\mu)$
induce identical predictive distributions. The observable $\mu$
is not distribution-invariant, while $|\mu|$ is.
Thus $|\mu|$ represents a well-defined element of $\mathcal{O}$.

\paragraph{Reduced-rank regression.}
Let $B = UV^\top$ with $U \in \mathbb{R}^{p \times r}$ and
$V \in \mathbb{R}^{q \times r}$. For any invertible
$R \in \mathrm{GL}(r)$,
\[
(U,V) \mapsto (UR, V R^{-T})
\]
leaves $B$ unchanged. Individual entries of $U$ and $V$
are not distribution-invariant, while functions of $B$
(e.g., its singular values) are.

\paragraph{Neural networks with permutation symmetry.}
Consider a one-hidden-layer neural network with $H$ hidden units,
\[
f_\theta(x)
=
\sum_{h=1}^H a_h \,\sigma(w_h^\top x),
\]
with parameters $\theta = \{(a_h, w_h)\}_{h=1}^H$.
For any permutation $\pi$ of $\{1,\dots,H\}$,
\[
(a_h, w_h) \mapsto (a_{\pi(h)}, w_{\pi(h)})
\]
leaves the induced predictive function unchanged. Thus parameter
configurations related by hidden-unit permutations lie in the same
equivalence class. Individual weight vectors $w_h$ are not distribution-invariant,
whereas functionals of the induced predictor $f_\theta$, such as its
$L^2$ norm on the data distribution or spectral properties of its
Jacobian, are invariant observables. 

Across these examples, the observable algebra identifies gauge-invariant
quantities that respond meaningfully under posterior tempering.
\section{Thermodynamic Response Functions}
\label{sec:responses}

Posterior tempering introduces a one-parameter family of probability measures
$\{\pi_\beta(\theta \mid D)\}_{\beta>0}$ on parameter space.
As a consequence, an integrable observable $f$ induces a scalar function of inverse
temperature via its tempered expectation $\mathbb{E}_\beta[f]$.
The goal of this section is to formalize how these expectations change with $\beta$
and to organize the resulting derivatives and variances into a hierarchy of
\emph{response functions}. All response functions considered below are defined for distribution-invariant observables, ensuring that the resulting quantities depend only on predictive distributions rather than parameterization.

We begin with a universal covariance identity, which is purely measure-theoretic
and holds without regularity assumptions beyond integrability.
We then show that when $f$ is distribution-invariant (in the sense of Section~\ref{sec:obs-algebra}),
the same identities may be expressed intrinsically on the model image
$\mathcal{M}=\Phi(\Theta)$, clarifying that response functions depend on predictive
structure rather than parameterization.

\subsection{Universal covariance identity}

Let
\[
\ell(\theta) := \log p(D \mid \theta)
\]
denote the log-likelihood.

\begin{proposition}[Covariance identity]
\label{prop:covariance}
For any observable $f$ integrable under $\pi_\beta$,
\[
\frac{d}{d\beta} \mathbb{E}_\beta[f]
=
\mathrm{Cov}_\beta(f,\ell).
\]
\end{proposition}

\begin{proof}
Write
\[
\mathbb{E}_\beta[f]
=
\frac{\int f(\theta)\pi(\theta)p(D\mid\theta)^\beta d\theta}
{\int \pi(\theta)p(D\mid\theta)^\beta d\theta}.
\]

Differentiating numerator and denominator with respect to $\beta$
and applying the quotient rule yields

\[
\frac{d}{d\beta}\mathbb{E}_\beta[f]
=
\mathbb{E}_\beta[f\ell] - \mathbb{E}_\beta[f]\mathbb{E}_\beta[\ell],
\]

which is precisely $\mathrm{Cov}_\beta(f,\ell)$.
\end{proof}

This identity shows that the sensitivity of any observable to the
tempering parameter is governed by its covariance with the log-likelihood.
Fluctuations in posterior geometry therefore determine how observables
respond to changes in effective temperature.

\paragraph{From parameter space to distribution space.}
Proposition~\ref{prop:covariance} is stated on parameter space $\Theta$, but in singular models many
distinct $\theta$ correspond to the same predictive distribution.
Section~\ref{sec:obs-algebra} formalized this through the model map $\Phi(\theta)=p(\cdot\mid\theta)$
and the model image $\mathcal{M}=\Phi(\Theta)$.
The next proposition records a useful consequence: when $f$ is distribution-invariant,
both the expectation $\mathbb{E}_\beta[f]$ and its $\beta$-derivative are determined
entirely by the pushforward posterior on $\mathcal{M}$. In this sense, the response identities are \emph{intrinsic} to the statistical model
rather than the chosen parameterization.

\begin{proposition}[Response identities descend to the model image]
\label{prop:descent}

Let $\Phi(\theta)=p(\cdot\mid\theta)$ and let $\mathcal{M}=\Phi(\Theta)$ be the model image.
If $f$ is distribution-invariant, i.e.\ $f=\psi\circ\Phi$ for some measurable $\psi:\mathcal{M}\to\mathbb{R}$,
then for every $\beta>0$,
\[
\mathbb{E}_\beta[f]
=
\int_{\mathcal{M}} \psi(p)\, d\widetilde{\pi}_\beta(p),
\qquad
\widetilde{\pi}_\beta := \Phi_\# \pi_\beta,
\]
and the thermodynamic response identity may be expressed entirely on $\mathcal{M}$:
\[
\frac{d}{d\beta}\mathbb{E}_\beta[f]
=
\mathrm{Cov}_\beta(f,\ell),
\qquad
\ell(\theta)=\log p(D\mid\theta),
\]
where both $f$ and $\ell$ are constant on equivalence classes and therefore define measurable functionals on $\mathcal{M}$.
\end{proposition}

\begin{proof}[Proof sketch]
The first identity is the definition of pushforward measure.
Since $\ell(\theta)=\log p(D\mid\theta)$ depends only on the induced predictive distribution, it is distribution-invariant.
Thus covariance and differentiation with respect to $\beta$ are well-defined at the level of $\mathcal{M}$.
\end{proof}

\paragraph{Packaging the consequences.}
Propositions~\ref{prop:covariance}--\ref{prop:descent} together show that tempering endows the posterior with a
linear-response structure: derivatives of expectations are covariances, and
curvatures are variances. It is therefore natural to treat $(\beta,f)\mapsto \mathbb{E}_\beta[f]$ as a basic
thermodynamic object and to define order parameters, susceptibilities, and
complexity measures as responses derived from this covariance structure. These identities hold for any integrable observable and therefore apply independently of whether the statistical model is regular or singular. 
The next theorem summarizes the resulting hierarchy. A schematic view of
these relationships is shown in Figure~\ref{fig:response_hierarchy}. 

\begin{theorem}[Thermodynamic response hierarchy]
\label{thm:response-hierarchy}
Let $\pi_\beta(\theta|D)$ be the tempered posterior and let
$f$ be any observable integrable under $\pi_\beta$. Then derivatives of tempered expectations are governed by
posterior fluctuation structure:
\[
\frac{d}{d\beta} \mathbb{E}_\beta[f]
=
\mathrm{Cov}_\beta(f,\ell).
\]
Consequently:
\begin{itemize}
\item Order parameters correspond to expectations
      $m(\beta)=\mathbb{E}_\beta[f]$.
\item Susceptibilities correspond to fluctuation magnitudes
      $\chi_f(\beta)=\beta\mathrm{Var}_\beta(f)$.
\item Curvature of the log-partition function satisfies
\[
\frac{d^2}{d\beta^2}\log Z(\beta)
=
\mathrm{Var}_\beta(\ell).
\]
\end{itemize}
Thus all response functions of the tempered posterior are generated
by posterior covariance structure. Moreover, higher-order derivatives generate higher-order cumulants of the
log-likelihood under the tempered posterior.
\end{theorem}

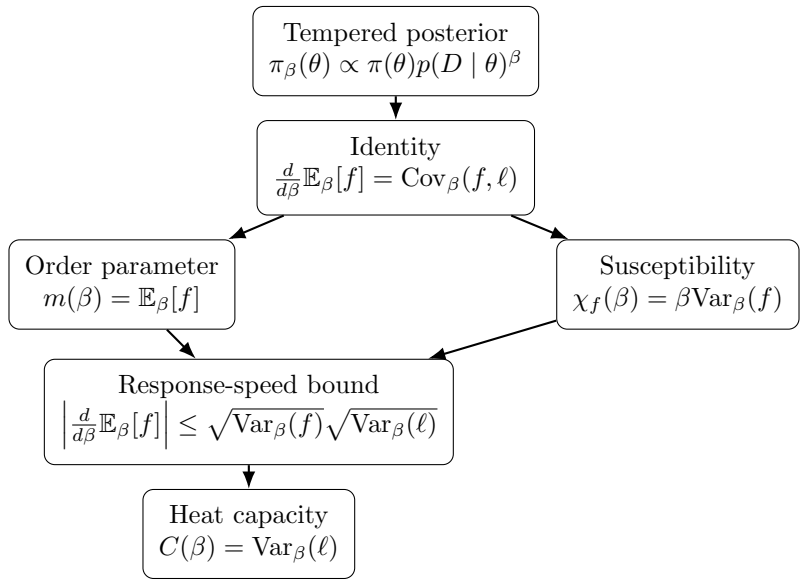
\begin{figure}[t]
\centering
\begin{tikzpicture}[
  box/.style={draw, rounded corners, align=center, inner sep=6pt},
  arr/.style={-Latex, thick}
]
\node[box] (temp) {Tempered posterior\\ $\pi_\beta(\theta)\propto \pi(\theta)p(D\mid\theta)^\beta$};

\node[box, below=9pt of temp] (id) {Identity\\ $\frac{d}{d\beta}\mathbb{E}_\beta[f]=\mathrm{Cov}_\beta(f,\ell)$};

\node[box, below left=9pt and 8pt of id] (op) {Order parameter\\ $m(\beta)=\mathbb{E}_\beta[f]$};
\node[box, below right=9pt and 8pt of id] (sus) {Susceptibility\\ $\chi_f(\beta)=\beta\mathrm{Var}_\beta(f)$};

\node[box, below=11pt of op, xshift=48pt] (bound) {Response-speed bound\\
$\Big|\frac{d}{d\beta}\mathbb{E}_\beta[f]\Big|\le \sqrt{\mathrm{Var}_\beta(f)}\sqrt{\mathrm{Var}_\beta(\ell)}$};

\node[box, below=9pt of bound] (cap) {Heat capacity\\ $C(\beta)=\mathrm{Var}_\beta(\ell)$};

\draw[arr] (temp) -- (id);
\draw[arr] (id) -- (op);
\draw[arr] (id) -- (sus);
\draw[arr] (op) -- (bound);
\draw[arr] (sus) -- (bound);
\draw[arr] (bound) -- (cap);
\end{tikzpicture}
\caption{Order parameters and susceptibilities arise from the covariance identity. The response-speed bound links rates of change to fluctuation magnitudes and connects naturally to heat capacity.}
\label{fig:response_hierarchy}
\end{figure}

\paragraph{Order parameters.}
We begin with first-order responses: expectations of distribution-invariant observables.
Given a distribution-invariant observable $f$, define the
\emph{order parameter}

\[
m(\beta) = \mathbb{E}_\beta[f].
\]
Order parameters track structural properties of the posterior
distribution as $\beta$ varies.
In singular models, these quantities often exhibit sharp changes
associated with structural reorganization of posterior mass.

\paragraph{Susceptibility.}
The variability of an observable under tempering defines a
\emph{susceptibility}
\[
\chi_f(\beta) = \beta\,\mathrm{Var}_\beta(f).
\]
Large susceptibility indicates that the observable fluctuates
strongly under the posterior, signaling sensitivity to changes
in effective temperature. In analogy with statistical physics, peaks in $\chi_f(\beta)$
indicate regions where the posterior reorganizes rapidly
with respect to $\beta$.

\paragraph{Response speed and fluctuation constraints.}
Order parameters describe structural properties of the posterior, while
susceptibilities quantify the magnitude of their fluctuations under tempering.
A natural next question is how rapidly these structural quantities can change
as the inverse temperature $\beta$ varies.
The covariance identity implies that such changes are driven by posterior
fluctuations of the log-likelihood.
Consequently the rate at which an observable can respond to changes in
$\beta$ is limited by the joint variability of the observable and the
log-likelihood. The following proposition provides a quantitative bound on
this response speed.

\begin{proposition}[Response-speed bound]
\label{prop:response-bound}
Let $f$ be an observable integrable under $\pi_\beta$. Then
\[
\left|\frac{d}{d\beta}\mathbb{E}_\beta[f]\right|
=
\left|\mathrm{Cov}_\beta(f,\ell)\right|
\le
\sqrt{\mathrm{Var}_\beta(f)}\;\sqrt{\mathrm{Var}_\beta(\ell)}.
\]
Equivalently, in terms of susceptibility $\chi_f(\beta)=\beta\,\mathrm{Var}_\beta(f)$ and
heat capacity $C(\beta)=\mathrm{Var}_\beta(\ell)$,
\[
\left|\frac{d}{d\beta}\mathbb{E}_\beta[f]\right|
\le
\sqrt{\frac{\chi_f(\beta)}{\beta}}\;\sqrt{C(\beta)}.
\]
\end{proposition}

\begin{proof}
By Proposition~\ref{prop:covariance},
$\frac{d}{d\beta}\mathbb{E}_\beta[f]=\mathrm{Cov}_\beta(f,\ell)$.
Cauchy--Schwarz yields \\$|\mathrm{Cov}_\beta(f,\ell)|
\le \sqrt{\mathrm{Var}_\beta(f)}\sqrt{\mathrm{Var}_\beta(\ell)}$.
\end{proof}

\noindent
Proposition~\ref{prop:response-bound} shows that rapid changes in order
parameters under tempering are necessarily accompanied by large
fluctuations in either the order parameter itself (large susceptibility)
or the log-likelihood (large heat capacity), providing a quantitative
basis for using peaks of these response functions as indicators of
structural reorganization.

\paragraph{Remark (Fluctuation–uncertainty relation).}
Proposition~\ref{prop:response-bound} resembles classical uncertainty
relations in Fourier analysis and statistical mechanics: the sensitivity
of an observable to changes in the control parameter $\beta$ is bounded
by the joint fluctuations of that observable and the log-likelihood.

\subsection{Complexity responses}

Certain response functions correspond directly to widely-used
Bayesian complexity measures.

\paragraph{WAIC complexity.}

Define

\[
p_{\mathrm{WAIC}}
=
\sum_{i=1}^n
\mathrm{Var}_\beta\!\left(\log p(y_i \mid \theta)\right).
\]

This quantity measures predictive variance across the posterior
and acts as an effective complexity correction in WAIC \cite{watanabe2010asymptotic, gelman2014understanding}. Because it is a variance of predictive log-likelihood terms, WAIC complexity
naturally appears as a second-order response quantity in the tempered posterior. 

Because $\log p(y_i \mid \theta)$ depends only on the induced predictive distribution, the WAIC complexity is invariant to
non-identifiable parameter directions. If two parameter values
$\theta$ and $\theta'$ induce the same predictive distribution,
then
\[
\log p(y_i \mid \theta) = \log p(y_i \mid \theta').
\]
Consequently posterior variability that occurs purely along
non-identifiable directions does not contribute to the WAIC
complexity term. In this sense WAIC measures predictive
fluctuation rather than parameter fluctuation. This invariance aligns naturally with the observable framework
of Section~\ref{sec:obs-algebra}: WAIC is determined entirely by
distribution-invariant observables.

\paragraph{Heat capacity.}

The variance of the log-likelihood

\[
C(\beta)
=
\mathrm{Var}_\beta(\ell)
\]

plays a role analogous to heat capacity in thermodynamics.
It measures fluctuations of the log-likelihood under the tempered
posterior. In statistical physics, heat capacity measures fluctuations of system
energy and typically peaks near phase transitions, where the system
reorganizes between competing configurations \cite{pathria2011statistical, callen1985thermodynamics}. An analogous phenomenon
occurs here. When posterior mass shifts between regions of parameter
space that induce qualitatively different predictive behavior, the
log-likelihood can fluctuate strongly under tempering. These
fluctuations produce peaks in $C(\beta)$. In singular models such reorganizations frequently occur near
degenerate parameter configurations where multiple predictive
structures coexist (for example, mixture component merging or rank
collapse). Heat-capacity peaks therefore act as empirical indicators
of structural transitions in posterior geometry.

\subsection{Unifying view}

Order parameters, susceptibilities, and complexity measures all arise
from derivatives of expectations under tempering. The covariance identity implies that these quantities ultimately derive
from the same fluctuation structure of the posterior distribution.
Within this framework, commonly used criteria such as WAIC and WBIC
can be interpreted as particular response functions evaluated at
specific temperatures. Taken together, these response functions distinguish between
predictive structural change and parameter redundancy.
Heat capacity detects reorganization of posterior support,
while WAIC measures predictive variability independent of
non-identifiable parameter directions.

\section{Singular Fluctuation as a Curvature Response}
\label{sec:nu}

The response-function framework developed above provides a natural
interpretation of classical quantities from singular learning theory.
In particular, the singular fluctuation can be viewed as a curvature
response of the tempered free energy.

\subsection{Free energy curvature}

Recall the tempered partition function
\[
Z(\beta)
=
\int \pi(\theta)\, p(D \mid \theta)^\beta d\theta,
\]
and the associated free energy
\[
F(\beta)
=
- \frac{1}{\beta} \log Z(\beta).
\]
Differentiating $\log Z(\beta)$ yields
\[
\frac{d}{d\beta}\log Z(\beta)
=
\mathbb{E}_\beta[\ell],
\qquad
\ell(\theta)=\log p(D\mid\theta).
\]
A second derivative gives
\[
\frac{d^2}{d\beta^2}\log Z(\beta)
=
\mathrm{Var}_\beta(\ell).
\]
Thus the variance of the log-likelihood corresponds to the curvature
of the log-partition function with respect to the tempering parameter.
In thermodynamic language this quantity plays the role of a heat
capacity.

Statistically, this curvature quantifies the extent to which
distinct parameter regions provide competing explanations
of the data. If the posterior concentrates around a single
predictive configuration, the log-likelihood is nearly constant
under $\pi_\beta$ and the curvature is small.
In contrast, when multiple structurally different explanations
coexist—such as different mixture allocations, ranks, or network
representations—the log-likelihood fluctuates substantially
across posterior samples, producing large curvature.

\subsection{Singular fluctuation}

Singular learning theory introduces the quantity
$\nu$, known as the \emph{singular fluctuation}, which governs
posterior variability of the log-likelihood and appears in
predictive risk expansions. In regular models, $\nu = d/2$, while in singular models $\nu$ generally differs from the
real log canonical threshold $\lambda$. 

Within the response framework, $\nu$ may be interpreted as the
leading-order contribution of log-likelihood fluctuations to the curvature of the tempered
free energy in the large-sample limit. Because these fluctuations correspond to curvature of the
log-partition function, the singular fluctuation plays the role
of a curvature response associated with the geometry of the
posterior distribution. This suggests that $\nu$ measures the predictive instability arising when multiple distinct parameter configurations provide competing explanations of the data.

\subsection{Relationship to WAIC and WBIC}

WAIC estimates predictive risk by measuring posterior
variance of pointwise log-likelihood terms \cite{gelman2014understanding}. As discussed in
Section~\ref{sec:responses}, this variance arises from the same
fluctuation structure that determines the curvature of the
free energy. WBIC evaluates the expectation of the log-likelihood under a
tempered posterior with inverse temperature $\beta_n = 1/ \log n$. This temperature corresponds to the regime in which the marginal
likelihood asymptotics of singular learning theory become visible.
From the response perspective, WBIC probes the free-energy landscape
at a temperature where curvature effects governed by $\lambda$
and $\nu$ determine model complexity. Together these relationships place WAIC, WBIC, and singular
fluctuation within a common response-function hierarchy
generated by posterior tempering, summarized in Table~\ref{tab:response_hierarchy}.

\begin{sidewaystable}
\centering
\begin{tabular}{llll}
\hline
Response level & Thermodynamic quantity & Statistical quantity & Interpretation \\
\hline
Order parameter & $m(\beta)=\mathbb{E}_\beta[f]$ & Structural observable & Effective model structure \\
Susceptibility & $\chi_f(\beta)=\beta\,\mathrm{Var}_\beta(f)$ & Structural fluctuation & Sensitivity of structure to tempering \\
Heat capacity & $C(\beta)=\mathrm{Var}_\beta(\ell)$ & Log-likelihood fluctuation & Posterior competition between explanations \\
Predictive response & $p_{\mathrm{WAIC}}$ & WAIC complexity & Predictive variability across posterior \\
Free-energy slope & RLCT $\lambda$ & Marginal likelihood scaling & Leading asymptotic complexity \\
Free-energy curvature & Singular fluctuation $\nu$ & Predictive risk correction & Curvature of tempered free energy \\
Special temperature probe & WBIC ($\beta_n=1/\log n$) & WBIC estimator & Marginal likelihood approximation regime \\
\hline
\end{tabular}
\caption{Hierarchy of thermodynamic response quantities generated by posterior tempering and their relationship to classical quantities from singular learning theory.}
\label{tab:response_hierarchy}
\end{sidewaystable}


\section{Experiments}
\label{sec:experiments}

We empirically investigate the response-function framework on three
canonical classes of singular statistical models:

\begin{enumerate}
\item Symmetry breaking in mixture models
\item Rank collapse in reduced-rank regression
\item Hidden-unit redundancy in neural networks
\end{enumerate}

These three examples correspond to three canonical mechanisms that generate singular statistical models: symmetry, algebraic constraints, and parameter redundancy. In each case we evaluate observables under a
tempered posterior

\[
\pi_\beta(\theta) \propto \pi(\theta)p(D|\theta)^\beta
\]

across a logarithmic grid of inverse temperatures $\beta$ spanning several orders of
magnitude to capture both prior-dominated and likelihood-dominated regimes. The experiments examine how order parameters, susceptibilities, and predictive complexity respond to posterior concentration. Posterior samples at each temperature are obtained using Hamiltonian Monte Carlo with temperature-specific likelihood scaling.

Across all experiments we report three quantities:

\begin{itemize}
\item Order parameter $m(\beta)$ describing the effective number of
active components in the model. In each experiment the order parameter is chosen to represent an observable
quantity that reflects the effective structural degrees of freedom used by the
model.
\item Susceptibility
\[
\chi(\beta) = \beta \, \mathrm{Var}_\beta(m)
\]
which measures fluctuations of the order parameter
\item WAIC complexity $p_{\mathrm{WAIC}}(\beta)$
\end{itemize}

The susceptibility plays the role of a thermodynamic response
function and detects transitions in posterior structure. 

Across the three experiments we observe a consistent response hierarchy
(Figures~\ref{fig:exp1_response}--\ref{fig:exp3_response}), in which
order parameters change smoothly with temperature, susceptibilities
exhibit peaks near structural transitions, and predictive complexity
declines as redundant structure collapses.

\begin{figure}[p]
\centering
\includegraphics[width=0.9\linewidth]{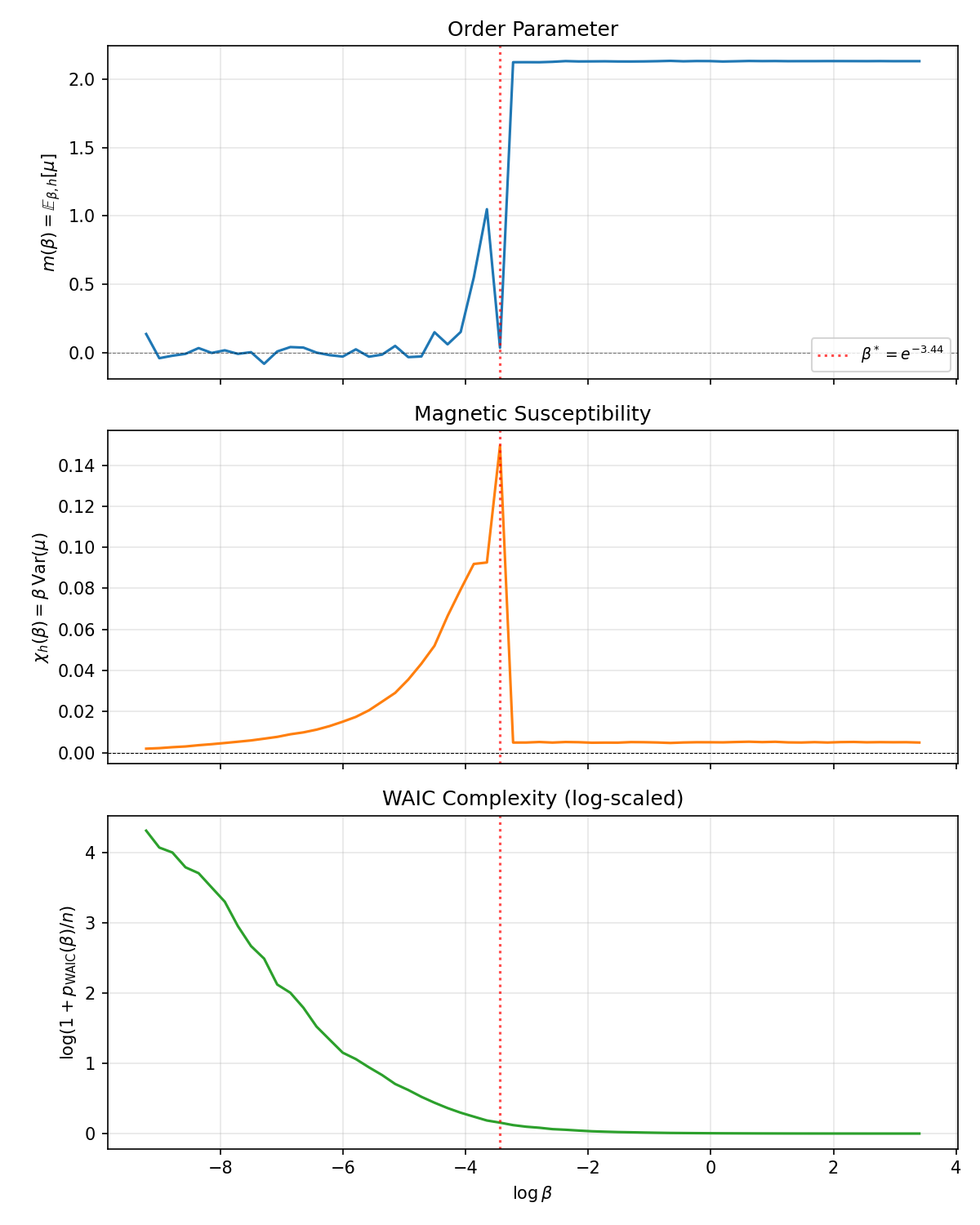}
\caption{
Response hierarchy for the mixture symmetry-breaking experiment. The top panel shows the order parameter 
$m(\beta) = \mathbb{E}_{\beta}[|\mu|]$, which measures the posterior preference for one component mean over the symmetric configuration.
At low inverse temperature the posterior explores both symmetric modes. As $\beta$ increases the posterior concentrates on one mode, producing spontaneous symmetry breaking. The middle panel shows the susceptibility $\chi(\beta) = \beta \mathrm{Var}(|\mu|)$,
which peaks near the transition where the posterior fluctuates between symmetric configurations. The bottom panel shows the WAIC complexity 
$\log(1+p_{\mathrm{WAIC}}(\beta)/n)$.
Predictive variance decreases as the posterior concentrates, indicating reduced predictive uncertainty once symmetry is broken. The vertical dashed line marks the temperature where susceptibility is maximal.
}
\label{fig:exp1_response}
\end{figure}

\subsection{Symmetry breaking in mixture models}

The first experiment studies a simple Gaussian mixture model with
exchangeable components.  Mixture models possess permutation
symmetries: swapping component labels leaves the likelihood
unchanged.  This symmetry produces singular posterior geometry
when components overlap.

As $\beta$ increases, posterior mass concentrates onto
configurations where only a subset of mixture components are
effectively active.  The order parameter measures the effective
number of occupied components.

The results show three regimes:

\begin{enumerate}
\item Low $\beta$: components contribute roughly equally
\item Intermediate $\beta$: symmetry begins to break
\item High $\beta$: one configuration dominates
\end{enumerate}

This behavior is illustrated in Figure~\ref{fig:exp1_response}, which shows
the order parameter, susceptibility, and WAIC complexity as functions of
the inverse temperature. The susceptibility exhibits a peak near the transition between the
symmetric and symmetry-broken regimes, indicating maximal
fluctuation of component allocations.

\subsection{Rank collapse in reduced-rank regression}

The second experiment examines reduced-rank regression (RRR), a
classical singular model where the parameter space contains matrices
of bounded rank. Let $Y = XB + \epsilon$ where $B$ is constrained to have rank
$r$.  Singularities occur when singular values of $B$ approach zero,
causing directions in parameter space to become unidentifiable.

The order parameter is defined using the effective rank of the
estimated coefficient matrix.  As $\beta$ increases the posterior
prefers lower-rank structures consistent with the data. Again we observe a susceptibility peak at the point where rank reduction
occurs, as shown in Figure~\ref{fig:exp2_response}. This reflects strong posterior fluctuations
between models with different effective dimensionality. 

\begin{figure}[p]
\centering
\includegraphics[width=0.9\linewidth]{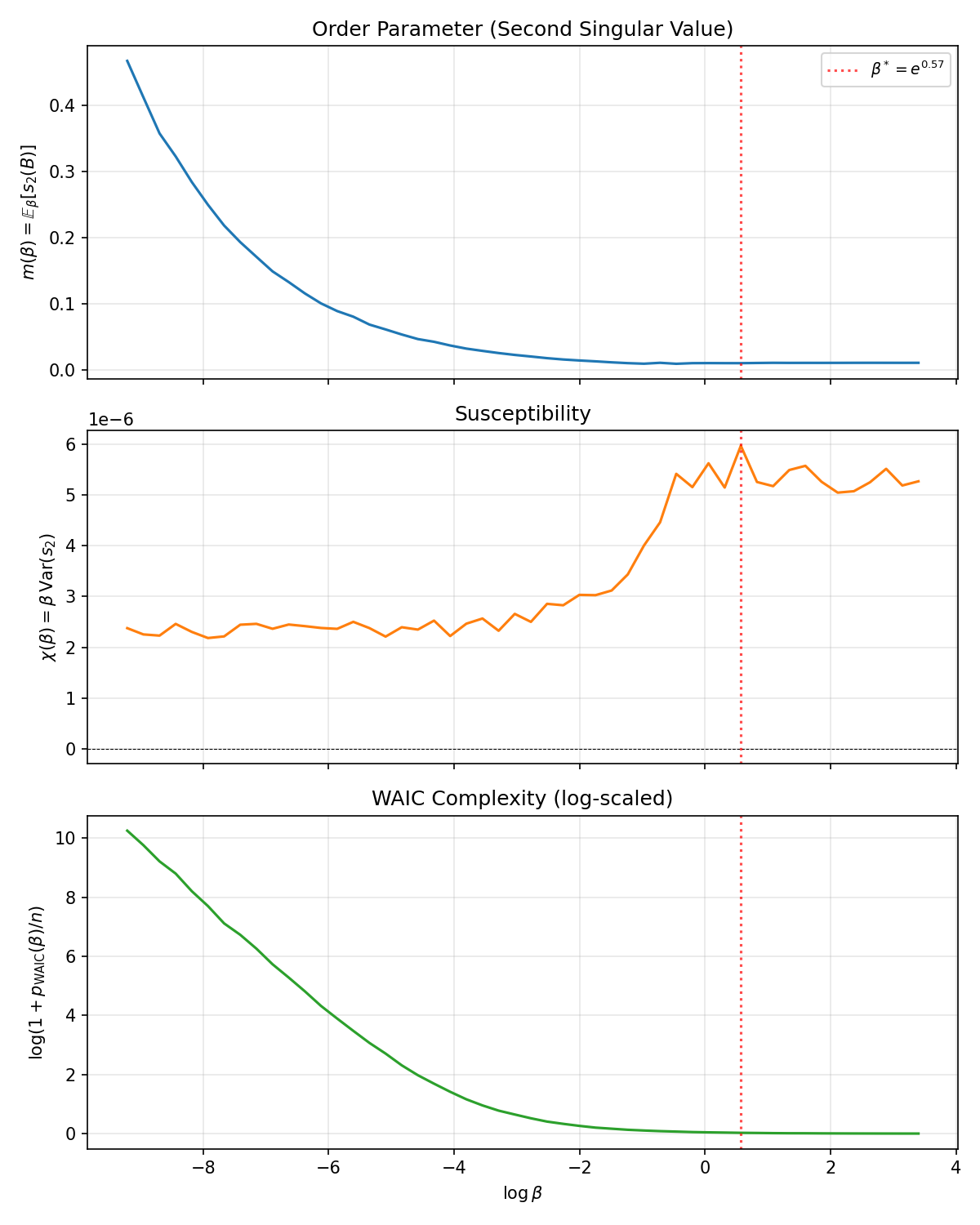}
\caption{
Response hierarchy for reduced-rank regression.
The order parameter $m(\beta)=\mathbb{E}_\beta[s_2(B)]$
tracks the second singular value of the regression matrix.
As $\beta$ increases, posterior concentration drives the second singular value toward zero,
indicating collapse to a lower-rank model. The susceptibility 
$\chi(\beta)=\beta\mathrm{Var}(s_2)$
measures fluctuations in the effective rank and peaks near the temperature where rank collapse occurs. The WAIC complexity decreases as the posterior eliminates redundant directions in parameter space.
The alignment between susceptibility and predictive complexity illustrates how singular
structure controls predictive variability.
}
\label{fig:exp2_response}
\end{figure}

\subsection{Hidden unit collapse in neural networks}

The third experiment studies a two-layer neural network trained on a
teacher–student regression task.  The model is
\[
f(x) = \sum_{j=1}^H a_j \phi(w_j x + b_j) + c .
\]
Neural networks exhibit several sources of singularity: permutation symmetry of hidden units, scaling degeneracies between $a_j$ and $w_j$, and redundancy of hidden units \cite{amari1998natural, watanabe2001algebraic}. To characterize active units we define amplitudes

\[
s_j = |a_j||w_j|
\]
and compute an effective number of active units
\[
N_{\mathrm{eff}}
=
\frac{\left(\sum_j s_j\right)^2}{\sum_j s_j^2}.
\]
This quantity equals approximately the number of hidden units that
meaningfully contribute to the network output. The susceptibility is computed as

\[
\chi(\beta)
=
\beta \, \mathrm{Var}_\beta(N_{\mathrm{eff}}).
\]

The resulting response hierarchy is shown in Figure~\ref{fig:exp3_response}.
As $\beta$ increases redundant units collapse and the effective
number of active units decreases.  The susceptibility exhibits a
pronounced peak where multiple configurations with different
numbers of active units coexist. The WAIC complexity follows a similar trend, confirming that
predictive variability is highest in the same region where order
parameter fluctuations are largest.

\begin{figure}[p]
\centering
\includegraphics[width=0.9\linewidth]{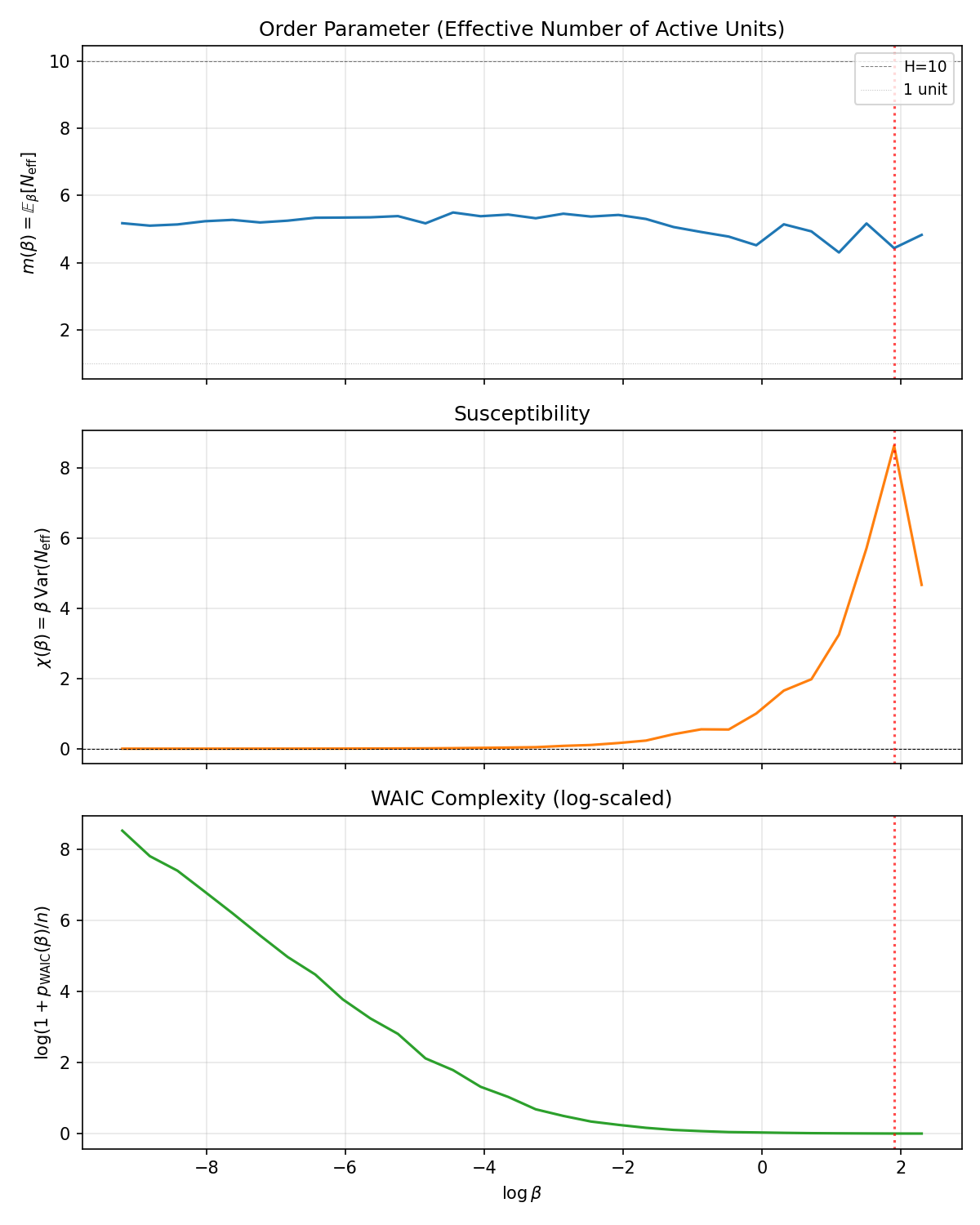}
\caption{
Response hierarchy for the neural network hidden-unit collapse experiment.
The order parameter 
$m(\beta)=\mathbb{E}_\beta[N_{\mathrm{eff}}]$
measures the effective number of active hidden units. Although the network contains $H=10$ units, the posterior favors a smaller effective number as $\beta$ increases.
Redundant hidden units become inactive due to symmetry and scaling degeneracies. The susceptibility
$\chi(\beta)=\beta\mathrm{Var}(N_{\mathrm{eff}})$
peaks when multiple configurations with different numbers of active units coexist.
This region corresponds to maximal posterior uncertainty over network representations. The WAIC complexity decreases as redundant units collapse, indicating that predictive uncertainty is highest when the network's internal representation is unstable.
}
\label{fig:exp3_response}
\end{figure}

\subsection{Summary of empirical findings}

Across all three models we observe the same qualitative behavior:

\begin{enumerate}
\item Order parameters reveal how many effective degrees of freedom
are active in the posterior.
\item Susceptibilities peak near transitions between different
posterior structures.
\item WAIC complexity tracks predictive variability and aligns with
regions of large susceptibility.
\end{enumerate}

These results support the interpretation that information criteria
such as WAIC correspond to response functions of the tempered
posterior distribution.  Singular learning phenomena therefore
appear naturally within a thermodynamic response hierarchy.

\section{Discussion}\label{sec:discussion}

\subsection{Thermodynamic interpretation of singular learning}

The results suggest that posterior geometry in singular statistical models
exhibits a thermodynamic response structure. Across three qualitatively different singular mechanisms—symmetry
breaking in mixture models, rank collapse in reduced-rank regression,
and hidden-unit redundancy in neural networks—we observe the same
qualitative behavior. As the inverse temperature $\beta$ increases, posterior mass
concentrates onto lower-dimensional regions of parameter space where
the model effectively uses fewer degrees of freedom.
Near these transitions the posterior fluctuates between competing
configurations, producing peaks in the susceptibility
$\chi(\beta)=\beta\mathrm{Var}_\beta(m)$. This behavior closely parallels phase transitions in statistical
physics, where macroscopic observables respond sharply to changes in
temperature.  In the statistical setting the control parameter
$\beta$ governs posterior concentration rather than physical energy,
but the resulting response structure appears analogous.

Importantly, the response framework does not rely on asymptotic expansions or
model-specific algebraic analysis. Instead, it identifies observable quantities
whose behavior under tempering reflects the same geometric structure that
drives singular learning theory. In this sense, response functions provide a
finite-sample diagnostic analogue of asymptotic invariants such as the RLCT
and singular fluctuation.

\subsection{Information criteria as response functions}

The experiments also suggest a thermodynamic interpretation of
predictive information criteria. The WAIC complexity term measures the posterior variance of
pointwise log likelihood contributions and therefore quantifies
fluctuations in predictive performance. Empirically we observe that WAIC complexity changes most rapidly near
the same temperatures where the susceptibility of structural order
parameters is largest. This supports the interpretation that predictive information
criteria act as response functions of the tempered posterior
distribution. Under this view, quantities such as WAIC and WBIC describe how
predictive performance responds to changes in posterior
concentration.

\subsection{Observables and identifiable structure}

A central motivation for the observable framework introduced in
Section~\ref{sec:obs-algebra} is that singular models contain
directions in parameter space that do not correspond to distinct
probability distributions. Permutation symmetries, scaling degeneracies, and redundant
parameterizations all introduce gauge directions that are invisible
to prediction. Observables correspond to functions defined on the space of
probability distributions rather than the parameter space itself.
Because they depend only on the induced predictive distribution,
observables are invariant to these unidentifiable directions.
This perspective clarifies why quantities such as WAIC remain well
defined even when the Fisher information matrix is singular.

\subsection{Practical implications}

Viewing singular models through the lens of response functions may
also provide practical insights. Susceptibility peaks identify regions where posterior geometry
changes rapidly, which may correspond to model selection boundaries
or transitions between effective model structures. For example, the hidden-unit collapse experiment shows that the
effective number of active units fluctuates strongly near the
transition where redundant units become inactive.
Such behavior may provide a useful diagnostic for understanding
capacity and redundancy in overparameterized models.

\subsection{Limitations}

The experiments presented here are illustrative rather than
exhaustive. They demonstrate consistent qualitative behavior across several
canonical singular models but do not yet provide a full theoretical
characterization of the response functions. In addition, exploring tempered posteriors across a wide range of
temperatures can be computationally expensive for large models.
Developing more efficient estimators of these response quantities
remains an open problem.

\subsection{Future directions}

Several theoretical directions emerge from this perspective. First, the covariance identities underlying the response functions
suggest the existence of a more general response operator acting on
observables. Understanding the spectral properties of this operator may reveal
how posterior geometry organizes singular models. Second, the analogy with statistical physics suggests that
renormalization group ideas may provide insight into how effective
model structure changes with scale \cite{wilson1975renormalization}. Finally, extending this framework to uncertainty quantification in
singular models may provide new tools for understanding predictive
variability in overparameterized systems.

\section{Conclusion}

This paper proposes a thermodynamic perspective on singular
statistical models based on observable quantities defined on the
space of predictive distributions.  By focusing on functions that are
invariant to unidentifiable parameter directions, we obtain a
representation of posterior structure that depends only on the
induced probability model rather than the specific parameterization.

Within this observable framework, tempering the posterior with an
inverse temperature parameter $\beta$ generates a hierarchy of
response functions analogous to those studied in statistical
physics.  Order parameters describe effective model structure,
susceptibilities measure fluctuations in that structure, and
predictive information criteria such as WAIC emerge naturally as
response quantities associated with predictive variance.

Experiments on mixture models, reduced-rank regression, and neural
networks demonstrate consistent qualitative behavior across
different classes of singular models.  In each case, transitions in
posterior geometry appear as peaks in susceptibility and changes in
predictive complexity.

These observations suggest that thermodynamic response theory may
provide a useful organizing framework for understanding singular
learning phenomena in modern statistical models. In particular, response functions offer a finite-sample diagnostic
counterpart to the asymptotic invariants of singular learning theory. More broadly, these results suggest that the thermodynamic response structure induced by posterior tempering may provide a practical route for probing the geometry of singular statistical models beyond asymptotic theory.

\section*{AI Acknowledgement}

The authors used a large language model (ChatGPT) as a writing and editing
assistant during the preparation of this manuscript. The system was used to
help organize outlines, improve exposition, and edit draft text. All
scientific ideas, mathematical arguments, experimental design, code, and
interpretations were developed and verified by the authors.

\bibliographystyle{plain}
\bibliography{refs}

\end{document}